\documentclass[letterpaper]{article}
\usepackage{proceed2e}
\usepackage[margin=1in]{geometry}
\usepackage{times}
\usepackage{graphicx}
\usepackage{subfigure} 
\usepackage[round]{natbib}
\usepackage{algorithm}
\usepackage{algorithmic}
\usepackage{hyperref}
\usepackage{graphicx}
\usepackage{listings}
\usepackage[utf8]{inputenc}
\usepackage{amsmath}
\usepackage{amsfonts}
\usepackage{amsthm}
\usepackage{amssymb}
\usepackage{physics}
\usepackage{color}

\newcommand{\KL}[2]{\text{KL}(#1 \, || \, #2)}
\newcommand{\btheta}{\boldsymbol\theta}
\newcommand{\bxi}{{\boldsymbol\xi}}
\newcommand{\bmu}{{\boldsymbol\mu}}
\newcommand{\bSigma}{{\boldsymbol\Sigma}}

\newcommand{\x}{\mathbf{x}}
\newcommand{\w}{\mathbf{w}}

\newcommand{\z}{\mathbf{z}}

\newcommand{\dataset}{\mathcal{D}}
\newcommand{\im}{\text{im}}

\theoremstyle{plain}
\newtheorem{theo}{Theorem}
\newtheorem{defi}{Definition}

\title{Differentially Private Variational Inference for Non-conjugate Models}

\author{\textbf{Joonas Jälkö}$^1$, \textbf{Onur Dikmen}$^1$ and \textbf{Antti Honkela}$^{1,2,3}$\\
  {$^1$ Helsinki Institute for Information Technology (HIIT), Department of Computer Science}\\
  {$^2$ Department of Mathematics and Statistics} \hspace{2em}
  {$^3$ Department of Public Health}\\
  {University of Helsinki}}

\begin{document}
\maketitle


\begin{abstract}
  Many machine learning applications are based on data collected from
  people, such as their tastes and behaviour as well as biological
  traits and genetic data. Regardless of how important the application
  might be, one has to make sure individuals' identities or the
  privacy of the data are not compromised in the analysis.
  Differential privacy constitutes a powerful framework that prevents
  breaching of data subject privacy from the output of a computation.
  Differentially private versions of many important Bayesian inference
  methods have been proposed, but there is a lack of an efficient unified
  approach applicable to arbitrary models. In this contribution, we
  propose a differentially private variational inference method with a
  very wide applicability. It is built on top of doubly stochastic
  variational inference, a recent advance
  which provides a variational solution to a large class of models.
  We add differential privacy into doubly stochastic variational
  inference by clipping and perturbing
  the gradients. The algorithm is made more efficient through privacy
  amplification from subsampling. We demonstrate the
  method can reach an accuracy close to non-private level under
  reasonably strong privacy guarantees, clearly improving over
  previous sampling-based alternatives especially in the strong
  privacy regime.
\end{abstract}


\section{INTRODUCTION}
Using more data usually leads to better generalisation and accuracy in
machine learning. With more people getting more tightly involved in
the ubiquitous data collection, privacy concerns related to the data
are becoming more important. People will be much more willing to
contribute their data if they can be sure that the privacy of their
data can be protected.

Differential privacy (DP) \citep{dwork_et_al_2006, DworkRoth} is a
strong framework with strict privacy guarantees against attacks from
adversaries with side information. The main principle is
that the output of an algorithm (such as a query or an estimator) should
not change much if the data for one individual are modified or deleted.
This can be accomplished through adding stochasticity at different
levels of the estimation process, such as adding noise to data itself
(input perturbation), changing the objective function to be optimised
or how it is optimised (objective perturbation), releasing the
estimates after adding noise (output perturbation) or by sampling from
a distribution based on utility or goodness of the alternatives (exponential
mechanism).

A lot of ground-breaking work has been done on privacy-preserving
versions of standard machine learning approaches, such as
objective-perturbation-based logistic regression
\citep{chaudhuri08privacy}, regression using functional
mechanism~\citep{zhang12functional} to name a few. Privacy-preserving
Bayesian inference
(e.g.~\citep{williams10probabilistic,zhang14privbayes}) has only
recently started attracting more interest.
The result of \citet{Dimitrakakis2014} showing that the posterior
distribution is under certain assumptions differentially private is
mathematically elegant, but does not lead to practically useful
algorithms.  Methods based on this approach suffer from the major
weakness that the privacy guarantees are only valid for samples drawn
from the exact posterior which is usually impossible to guarantee in
practice.  Methods based on perturbation of data sufficient statistics
\citep{zhang16differential,foulds16theory,Honkela2016} are
asymptotically efficient, but they are only applicable to exponential
family models which limits their usefulness.  The sufficient statistic
perturbation approach was recently also applied to variational
inference~\citep{Park2016}, which is again applicable to models
where non-private inference can be performed by accessing sufficient
statistics.

General differentially private Bayesian inference can be realised most
easily using the gradient perturbation mechanism.  This was first
proposed by \citet{wang15privacy}, who combine differential privacy by
gradient perturbation with stochastic gradient Markov chain Monte
Carlo (MCMC) sampling.  This approach works in principle for arbitrary
models, but because of the gradient perturbation mechanism each MCMC
iteration will consume some privacy budget, hence severely limiting
the number of iterations that can be run which can cause difficulties
with the convergence of the sampler.

Our goal in this work is to apply the gradient perturbation mechanism
to devise a generic differentially private variational inference
method. Variational inference seems preferable to stochastic gradient
MCMC here because a good optimiser should be able to make better use
of the limited gradient evaluations and the variational approximation
provides a very efficient summary of the posterior.
The recently proposed doubly stochastic
variational inference~\citet{Titsias2014} and the further streamlined
automatic differentiation variational inference (ADVI) method
\citep{Kucukelbir2017} provide a generic variational inference method
also applicable to non-conjugate models.
These approaches apply a series of transformations and approximations so
that the variational distributions are Gaussian and can be optimised
by stochastic gradient ascent. Here, we propose differentially
private variational inference (DPVI) based on gradient clipping and
perturbation as well as double stochasticity. We make a thorough case study on the Bayesian logistic
regression model with comparisons to the non-private case under
different design decisions for DPVI. We also test the performance 
of DPVI with a Gaussian mixture model.


\section{BACKGROUND}
\subsection{DIFFERENTIAL PRIVACY}
Differential privacy (DP) \citep{dwork_et_al_2006,DworkRoth} is a framework
that provides mathematical formulation for privacy that enables
proving strong privacy guarantees.
\begin{defi}[$\epsilon$-Differential privacy]
\label{eps-dp}
A randomised algorithm $\mathcal{A}$ is $\epsilon$-differentially
private if for all pairs of adjacent data sets, i.e., differing only in
one data sample, $x,x'$, and for all sets $S \subset \im(\mathcal{A})$
\begin{equation*}
\Pr(\mathcal{A}(x) \in S) \leq e^\epsilon \Pr(\mathcal{A}(x') \in S).
\end{equation*}
\end{defi}
There are two different variants depending on which data sets are
considered adjacent: in \emph{unbounded DP} data sets $x,x'$ are
adjacent if $x'$ can be obtained from $x$ by adding or removing an
entry, while in \emph{bounded DP} $x,x'$ are adjacent if they are of
equal size and equal in all but one of their elements
\citep{DworkRoth}.  The definition is symmetric in $x$ and $x'$ which
means that in practice the probabilities of obtaining a specific
output from either algorithm need to be similar.  The privacy
parameter $\epsilon$ measures the strength of the guarantee with
smaller values corresponding to stronger privacy.

$\epsilon$-DP defined above, also known as \emph{pure DP}, is
sometimes too inflexible and a relaxed version called $(\epsilon,
\delta)$-DP is often used instead.  It is defined as follows:
\begin{defi}[$(\epsilon,\delta)$-Differential privacy]
\label{ed-dp}
A randomised algorithm $\mathcal{A}$ is $(\epsilon,\delta)$-differentially private if for all pairs of adjacent data sets $x,x'$ and for every $S \subset \im(\mathcal{A})$
\begin{align*}
\Pr(\mathcal{A}(x) \in S) \leq e^\epsilon \Pr(\mathcal{A}(x') \in S) + \delta.
\end{align*}
\end{defi}
It can be shown that $(\epsilon, \delta)$-DP provides a probabilistic
$\epsilon$-DP guarantee with probability $1-\delta$ \citep{DworkRoth}.

\subsubsection{Gaussian mechanism}
There are many possibilities how to make algorithm differentially
private. In this paper we use \emph{objective perturbation}. We use
the \emph{Gaussian mechanism} as our method for perturbation.
\citet[Theorem 3.22]{DworkRoth} state that given query $f$ with
$\ell_2$-sensitivity of $\Delta_2(f)$, releasing $f(x)+\eta$, where
$\eta \sim N(0, \sigma^2)$, is $(\epsilon, \delta)$-DP when
\begin{equation}
  \label{eq:gauss_sigmabound}
  \sigma^2 > 2\ln(1.25/\delta)\Delta^2_2(f)/\epsilon^2.
\end{equation}
The important $\ell_2$-sensitivity of a query is defined as:
\begin{defi}[$\ell_2$-sensitivity]
Given two adjacent data sets $x, x'$, $\ell_2$-sensitivity of query $f$ is
\begin{align*}
\Delta_2 (f) = \sup_{\stackrel{x,x'}{||x-x'||=1}} || f(x)-f(x') ||_2.
\end{align*} 
\end{defi}

\subsubsection{Composition theorems}
\label{comp}

One of the very useful features of DP compared to many other privacy
formulations is that it provides a very natural way to study the
privacy loss incurred by repeated use of the same data set.  Using an
algorithm on a data set multiple times will weaken our privacy
guarantee because of the potential of each application to leak more
information.  The DP variational inference algorithm proposed in this
paper is iterative, so we need to use composition theorems to bound
the total privacy loss.

The simplest basic composition \cite{DworkRoth} shows that a $k$-fold
composition of an $(\epsilon, \delta)$-DP algorithm provides
$(k\epsilon, k\delta)$-DP. More generally releasing joint output of
$k$ algorithms $\mathcal{A}_i$ that are individually $(\epsilon_i,
\delta_i)$-DP will be $(\sum_{i=1}^k \epsilon_i, \sum_{i=1}^k
\delta_i)$-DP. Under pure $\epsilon$-DP when $\delta_1 = \dots =
\delta_k = 0$ this is the best known composition that yields a pure DP
algorithm.

Moving from the pure $\epsilon$-DP to general $(\epsilon, \delta)$-DP
allows a stronger result with a smaller $\epsilon$ at the expense of
having a larger total $\delta$ on the composition.  This trade-off
is characterised by the Advanced composition theorem of
\citet[][Theorem 3.20]{DworkRoth}, which becomes
very useful when we need to use data multiple times
\begin{theo}[Advanced composition theorem]
Given algorithm $\mathcal{A}$ that is $(\epsilon, \delta)$-DP and $\delta'>0$, $k$-fold composition of algorithm $\mathcal{A}$ is $(\epsilon_{tot}, \delta_{tot})$-DP with
\begin{align}
  \epsilon_{tot} &= \sqrt{2k\ln(1/\delta')}\epsilon + k\epsilon(e^\epsilon-1) \\
  \delta_{tot} &= k\delta + \delta'.
\end{align}
\end{theo}
The theorem states that with small loss in $\delta_{tot}$ and with
small enough $\epsilon$, we can provide more strict $\epsilon_{tot}$
than just summing the $\epsilon$.  This is obvious by looking at the
first order expansion for small $\epsilon$ of
\[ \epsilon_{tot} \approx \sqrt{2k\ln(1/\delta')}\epsilon +
k\epsilon^2. \]

\subsubsection{Privacy amplification}
We use a stochastic gradient algorithm that uses subsampled data while
learning, so we can make use of the amplifying effect of the
subsampling on privacy. This \textit{Privacy amplification theorem}
\citep{Li} states that if we run $(\epsilon, \delta)$-DP algorithm
$\mathcal{A}$ on randomly sampled subset of data with uniform sampling
probability $q>\delta$, privacy amplification theorem states that
the subsampled algorithm is $(\epsilon_{amp}, \delta_{amp})$-DP with
\begin{align}
  \epsilon_{amp} &= \min(\epsilon, \log(1+q(e^\epsilon-1))) \\
  \delta_{amp} &= q\delta,
\end{align}
assuming $\log(1+q(e^\epsilon-1)) < \epsilon$.

\subsubsection{Moments accountant}

The moments accountant proposed by \citet{Abadi2016} is a method to
accumulate the
privacy cost that provides a tighter bound for $\epsilon$ and $\delta$
than the previous composition approaches.
The moments accountant incorporates both the composition over
iterations and privacy amplification due to subsampling into a single
bound given by the following Theorem.
\begin{theo}
  There exist constants $c_1$ and $c_2$ so that given the sampling
  probability $q = L/N$ and the number of steps $T$, for any
  $\epsilon < c_1 q^2 T$, a DP stochastic gradient algorithm that
  clips the $\ell_2$ norm of gradients to $C$ and injects Gaussian
  noise with standard deviation $2 C \sigma$ to the gradients, is
  $(\epsilon, \delta)$-DP for any $\delta > 0$ under bounded DP
  if we choose
  \begin{equation}
    \label{eq:sigmabound}
    \sigma \ge c_2 \frac{q \sqrt{T \log(2/\delta)}}{\epsilon}.
  \end{equation}
\end{theo}
\begin{proof}
  \citet{Abadi2016} show that injecting gradient noise with standard
  deviation $C \sigma$ where $\sigma$ satisfies the inequality
  (\ref{eq:sigmabound}) yields an $(\epsilon, \frac{1}{2}\delta)$-DP
  algorithm under unbounded DP.  This implies that adding noise with
  standard deviation $2 C \sigma$ yields an
  $(\frac{1}{2}\epsilon, \frac{1}{2}\delta)$-DP algorithm under
  unbounded DP.

  This proves the theorem as any
  $(\frac{1}{2}\epsilon, \frac{1}{2}\delta)$ unbounded DP algorithm is
  an $(\epsilon, \delta)$ bounded DP algorithm.  This follows
  from the fact that the replacement of an element in the data set can
  be represented as a composition of one addition and one removal of
  an element.
\end{proof}
Similar bounds can also be derived using concentrated DP
\citep{Dwork2016,Bun2016}.

We use the implementation of \citet{Abadi2016} to compute the total
$\epsilon$ privacy cost with a given $\delta$-budget, standard
deviation $\sigma$ of noise applied in Gaussian mechanism and
subsampling ratio $q$.

In our experiments we report results using both the advanced
composition theorem with privacy amplification as well as the moments
accountant.

\subsection{VARIATIONAL BAYES}
Variational Bayes (VB) methods \citep{jordan99introduction} provide a
way to approximate the posterior distribution of latent variables in a
model when the true posterior is intractable. True posterior
$p(\btheta|\x)$ is approximated with a variational distribution
$q_\bxi(\btheta)$ that has a simpler form than the posterior, obtained
generally by removing some dependencies from the graphical model such
as the fully-factorised form
$q_\bxi(\btheta)=\prod_d q_{\bxi_d}(\theta_d)$. $\bxi$ are the
variational parameters and their optimal values $\bxi^*$ are obtained
through minimising the Kullback-Leibler (KL) divergence between
$q_\bxi(\btheta)$ and $p(\btheta|\x)$. This is also equivalent to
maximising the \emph{evidence lower bound} (ELBO)
\begin{align*}
\mathcal{L}(q_\bxi) &= \int q_\bxi(\btheta) \ln \left( \frac{p(\dataset, \btheta)}{q_\bxi(\btheta)} \right) \\ 
&=-\KL{q_\bxi(\btheta)}{p(\btheta)} + \sum_{i=1}^B \langle \ln p(x_i|\btheta) \rangle_{q_\bxi(\btheta)},
\end{align*} 
where $\langle \rangle_{q_\bxi(\btheta)}$ is an expectation taken w.r.t $q_\bxi(\btheta)$
and the observations $\dataset = \{\x_1, \dots, \x_N\}$ are
assumed to be exchangeable under our model.

When the model is in the conjugate exponential family
\citep{ghahramani2001propagation} and $q_\bxi(\btheta)$ is factorised,
the expectations that constitute $\mathcal{L}(q_\bxi)$ are
analytically available and each $\bxi_d$ is updated iteratively by
fixed point iterations. Most popular applications of VB fall into this
category, because handling of the more general case involves more
approximations, such as defining another level of lower bound to the ELBO
or estimating the expectations using Monte Carlo integration.

\subsubsection{Doubly stochastic variational inference}

An increasingly popular alternative approach is the doubly stochastic
variational inference framework proposed by \citet{Titsias2014}.  The
framework is based on stochastic gradient optimisation of the ELBO.
The expectation over $q_\bxi(\btheta)$ is evaluated using Monte Carlo
sampling.  Exchanging the order of integration (expectation) and
differentiation and using the reparametrisation trick to represent for
example samples from a Gaussian approximation
$q_{\bxi_i}(\btheta_i) = N(\btheta_i ; \bmu_i, \bSigma_i)$ as
$\btheta_i = \bmu_i + \bSigma_i^{1/2} \z, \; \z \sim N(0, I)$, it is
possible to obtain stochastic gradients of the ELBO which can be fed
to a standard stochastic gradient ascent (SGA) optimisation algorithm.

For models with exchangeable observations, the ELBO objective can be
broken down to a sum of terms for each observation:
\begin{align*}
\mathcal{L}(q_\bxi)
  &=-\KL{q_\bxi(\btheta)}{p(\btheta)} + \sum_{i=1}^N \langle \ln p(x_i|\btheta) \rangle_{q_\bxi(\btheta)} \\
  &= \sum_{i=1}^N \left( \langle \ln p(x_i|\btheta) \rangle_{q_\bxi(\btheta)} -\frac{1}{N} \KL{q_\bxi(\btheta)}{p(\btheta)} \right) \\
  & =: \sum_{i=1}^N \mathcal{L}_i(q_\bxi).
\end{align*} 
This allows considering mini batches of data at each iteration
to handle big data sets, which adds another level of stochasticity to
the algorithm.

The recently proposed Automatic Derivation Variational Inference (ADVI)
framework \citep{Kucukelbir2017} unifies
different classes of models through a transformation of variables and
optimises the ELBO using stochastic gradient ascent (SGA). Constrained
variables are transformed into unconstrained ones and their posterior
is approximated by Gaussian variational distributions, which can be a
product of independent Gaussians (mean-field) or larger multivariate
Gaussians. Expectations in the gradients are approximated using Monte
Carlo integration and the ELBO is optimised iteratively using SGA.


\section{DIFFERENTIALLY-PRIVATE VARIATIONAL INFERENCE}
Differentially-private variational inference (DPVI) is based on
clipping of the contributions of individual data samples to the
gradient, $g_t(x_i) = \nabla \mathcal{L}_i(q_{\bxi_t})$, at each
iteration $t$ of the stochastic
optimisation process and perturbing the total gradient.  The
algorithm is presented in Algorithm \ref{alg:dpvi}.  The algorithm is
very similar to the one used for deep learning by \citet{Abadi2016}.
DPVI can
be easily implemented using automatic differentiation software such as
Autograd, or incorporated even more easily into automatic inference
engines, such as the ADVI implementations in PyMC3
\citep{Salvatier2016} or Edward \citep{tran2017deep} which also
provide subsampling. Each
$g_t(x_i)$ is clipped in order to calculate gradient sensitivity.
Gradient contributions from all data samples in the mini batch are
summed and perturbed with Gaussian noise
$\mathcal{N}(0, 4c_t^2\sigma^2 \mathbf{I})$.

\begin{algorithm}[tb]
   \caption{DPVI}
   \label{alg:dpvi}
\begin{algorithmic}
   \STATE {\bfseries Input:} Data set $\dataset$, sampling probability $q$, number of iterations $T$, SGA step size $\eta$, Clipping threshold $c_t$ and initial values $\bxi_0$.
   \FOR{$t \in [T] $}
	\STATE	Pick random sample $U$ from $\dataset$ with sampling probability $q$\;
	\STATE	Calculate the gradient $g_t(x_i) = \nabla \mathcal{L}_i(q_{\bxi_t})$  for each $i \in U$\;
	\STATE	Clip and sum gradients:
	\STATE  $\tilde{g}_t(x_i) \leftarrow g_t(x_i) / \max(1, \frac{||g_t(x_i)||_2)}{c_t})$\;					
	\STATE	$\tilde{g}_t \leftarrow \sum_i \tilde{g}_t(x_i)$\;
	\STATE	Add noise:  $\tilde{g}_t \leftarrow \tilde{g}_t + \mathcal{N}(0, 4c_t^2\sigma^2 \mathbf{I})$\;
	\STATE	Update AdaGrad parameter. $G_t \leftarrow G_{t-1}+\tilde{g}_t^2$ \;
	\STATE	Ascent: $\boldsymbol\xi_{t} \leftarrow \boldsymbol\xi_{t-1} + \eta \tilde{g}_t/\sqrt{G_t}$	
   \ENDFOR
\end{algorithmic}
\end{algorithm}

The sampling frequency $q$ for subsampling within the data set, total
number of iterations $T$ and the variance $\sigma^2$ of Gaussian noise are
important design parameters that determine the privacy cost. $c_t$ is
chosen before learning, and does not need to be constant. After
clipping $||g_t(x_i)||_2 \leq c_t, \, \forall i \in U$. Clipping
gradients too much will affect accuracy, but on the other hand large
clipping threshold will cause large amount of noise to sum of
gradients. Parameter $q$ determines how large subsample of the
training data we use to for gradient ascent. Small $q$ values enable
privacy amplification but may need a need larger $T$. For a very small
$q$ when the mini batches consist of just a few samples, the added
noise will dominate over the gradient signal and the optimisation will
fail. While in our experiments $q$ was fixed, we could also alter the
$q$ during iteration.

\subsection{MODELS WITH LATENT VARIABLES}

The simple approach in Algorithm \ref{alg:dpvi} will not work well for
models with latent variables.  This is because the main gradient
contributions to latent variables come from only a single data point,
and the amount of noise that would need to be injected to mask the
contribution of this point as needed by DP would make the gradient
effectively useless.

One way to deal with the problem is to take the EM algorithm view
\citep{Dempster1977} of latent variables as a hidden part of a larger
complete data set and apply the DP protection to summaries computed
from the complete data set.  In this approach, which was also used by
\citet{Park2016}, no noise would be injected to the updates of the
latent variables but the latent variables would never be released.

An alternative potentially easier way to avoid this problem is to
marginalise out the latent variables if the model allows this.  As the
DPVI framework works for arbitrary likelihoods we can easily perform
inference even for complicated marginalised likelihoods.  This is a
clear advantage over the VIPS framework of \citet{Park2016} which
requires conjugate exponential family models.

\subsection{SELECTING THE ALGORITHM HYPERPARAMETERS}
\label{sec:select-algor-hyperp}

The DPVI algorithm depends on a number of parameters, the most
important of which are the gradient clipping threshold $c_t$, the data
subsampling ratio $q$ and the number of iterations $T$.  Together
these define the total privacy cost of the algorithm, but it is not
obvious how to find the optimal combination of these under a fixed
privacy budget.  Unfortunately the standard machine learning
hyperparameter adaptation approach of optimising the performance on a
validation set is not directly applicable, as every test run would
consume some of the privacy budget.  Developing good heuristics for
parameter tuning is thus important for practical application of the
method.

Out of these parameters, the subsampling ratio $q$ seems easiest to
interpret.  The gradient that is perturbed in Algorithm~\ref{alg:dpvi}
is a sum over $qN$ samples in the mini batch.  Similarly the standard
deviation of the noise injected with the moments accountant in
Eq.~(\ref{eq:sigmabound}) scales linearly with $q$.  Thus the
signal-to-noise ratio for the gradients will be independent of $q$ and
$q$ can be chosen to minimise the number of iterations $T$.

The number of iterations $T$ is potentially more difficult to
determine as it needs to be sufficient but not too large.  The moments
accountant is somewhat forgiving here as its privacy cost increases
only in proportion to $\sqrt{T}$.  In practice one may need to simply
pick $T$ believed to be sufficiently large and hope for the best.
Poor results in the end likely indicate that the number of samples in
the data set may be insufficient for good results at the given level
of privacy.

The gradient clipping threshold $c_t$ may be the most difficult
parameter to tune as that depends strongly on the details of the
model.  Fortunately our results do not seem overly sensitive to using
the precisely optimal value of $c_t$.  Developing good heuristics for
choosing $c_t$ is an important objective for future research.  Still,
the same problem is shared by every DL method based on gradient
perturbation including the deep learning work of \citet{Abadi2016} and
the DP stochastic gradient MCMC methods of \citet{wang15privacy}.  In
the case of stochastic gradient MCMC this comes up through selecting a
bound on the parameters to bound the Lipschitz constant appearing in
the algorithm.  A global Lipschitz constant for the ELBO would
naturally translate to a $c_t$ guaranteed not to distort the
gradients, but as noted by \citet{Abadi2016}, it may actually be good
to clip the gradients to make the method more robust against outliers.


\section{EXPERIMENTS}

\subsection{LOGISTIC REGRESSION}

\begin{figure*}[tb]
  \centering
  \includegraphics[width=\columnwidth]{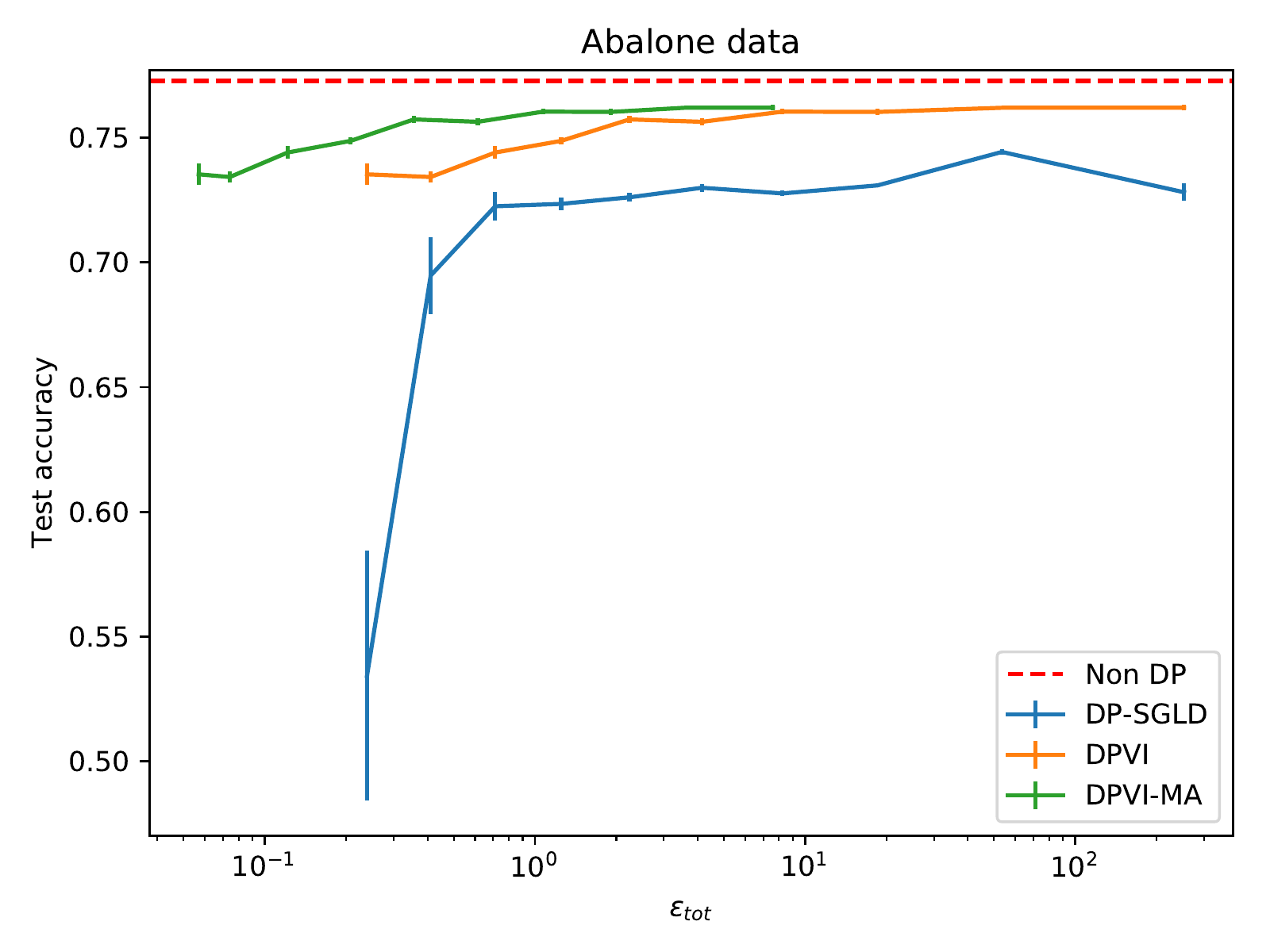} \hfill
  \includegraphics[width=\columnwidth]{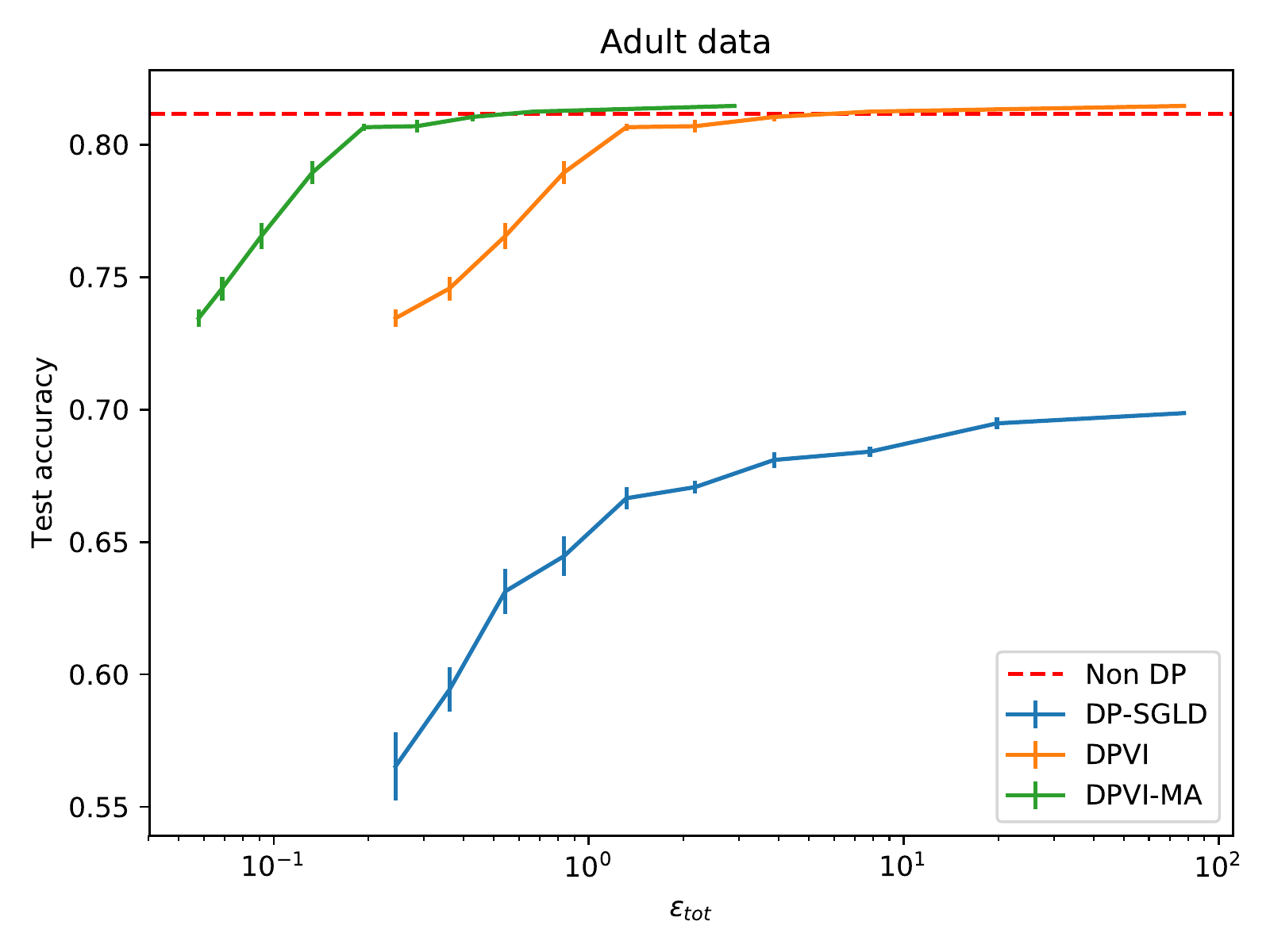}
  \caption{Comparison of binary classification accuracies using the
    Abalone data set (left) and the Adult data set (right).
    The figure shows test set classification
    accuracies of non-private logistic regression, two variants of
    DPVI with the moments accountant and advanced composition
    accounting and DP-SGLD of \citet{wang15privacy}.  
    The curve shows the mean of 10 runs
    of both algorithms with error bars denoting the standard error of
    the mean.}
  \label{fig:dpvi_vs_sgld1}
\end{figure*}

We tested DPVI with two different learning tasks.  
Lets first consider model of logistic regression using the Abalone
and Adult data sets
from the UCI
Machine Learning Repository \citep{Lichman} for the binary
classification task. Our model is:
\begin{align*}
P(y | \x, \w) &= \sigma(y\w^T\x) \\
p(\w) & = N(\w ; \w_0, \mathbf{S}_0),
\end{align*}
where $\sigma(x) = 1/(1+\exp(-x))$.

For Abalone, individuals were divided into two classes based
on whether individual had less or more than 10 rings. The data set
consisted of 4177 samples with 8 attributes. We learned a posterior
approximation
for $\w$ using ADVI with SGA using Adagrad optimiser \citep{Duchi} and sampling 
ratio $q=0.02$.
The posterior approximation $q(\w)$ was Gaussian with a diagonal
covariance.
Classification was done using an additional Laplace approximation. Before training, features of the data set were
normalised by subtracting feature mean and dividing by feature
standard deviation. Training was done with $80\%$ of data.

The other classification dataset ``Adult'' that we used with logistic
regression consisted of
48842 samples with 14 attributes. Our classification task was to predict
whether or not an individual's annual income exceeded \$50K.
The data were preprocessed similarly as in Abalone: we subtracted the
feature mean and divided by the standard deviation of each feature.
We again used $80\%$
of the data for training the model.

We first compared the classification accuracy of models learned using
two variants of DPVI with the moments accountant and advanced
composition accounting as well as DP-SGLD of \citet{wang15privacy}.
The classification results for Abalone and Adult are shown in
Fig.~\ref{fig:dpvi_vs_sgld1}. We used $q=0.05$ in Abalone corresponding
to mini batches of 167 points and $q=0.005$ in Adult corresponding
to mini batches of 195 points.
With Abalone the algorithm was run for 1000 iterations and
with Adult for 2000 iterations. Clipping threshold were 5 for Abalone
and 75 for Adult. Both results clearly show that even
under comparable advanced composition accounting used by DP-SGLD, DPVI
consistently yields significantly higher classification accuracy at a
comparable level of privacy.  Using the moments accountant further
helps in obtaining even more accurate results at comparable level of
privacy.  DPVI with the moments accountant can reach classification
accuracy very close to the non-private level already for $\epsilon < 0.5$
for both data sets.

\subsubsection{The effect of algorithm hyperparameters.}

We further tested how changing the different hyperparameters of the
algorithm discussed in Sec.~\ref{sec:select-algor-hyperp} affects the
test set classification accuracy of models learned with DPVI with the
moments accountant.

\begin{figure*}[tb]
  \centering
  \includegraphics[width=\columnwidth]{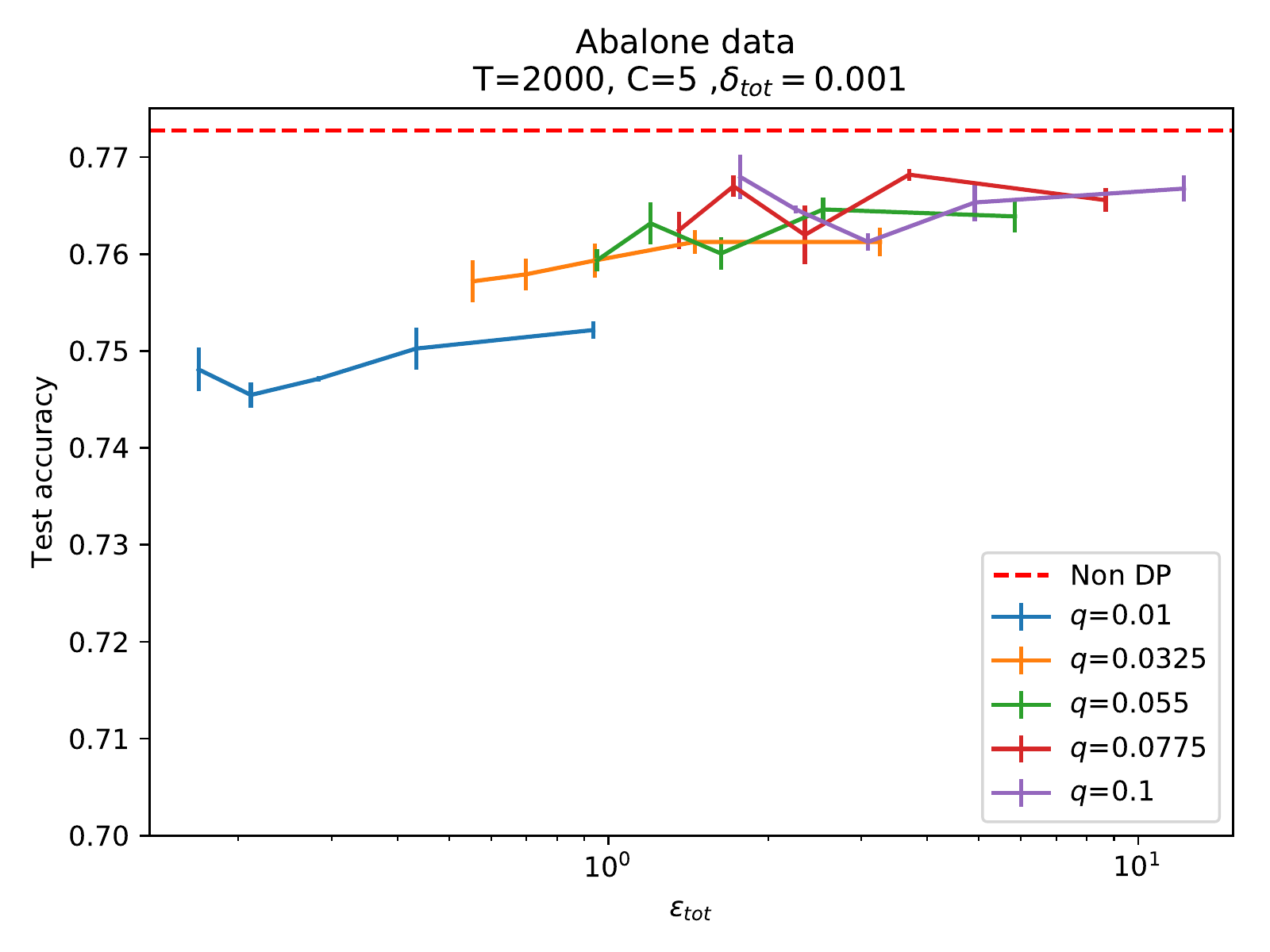}\hfill
  \includegraphics[width=\columnwidth]{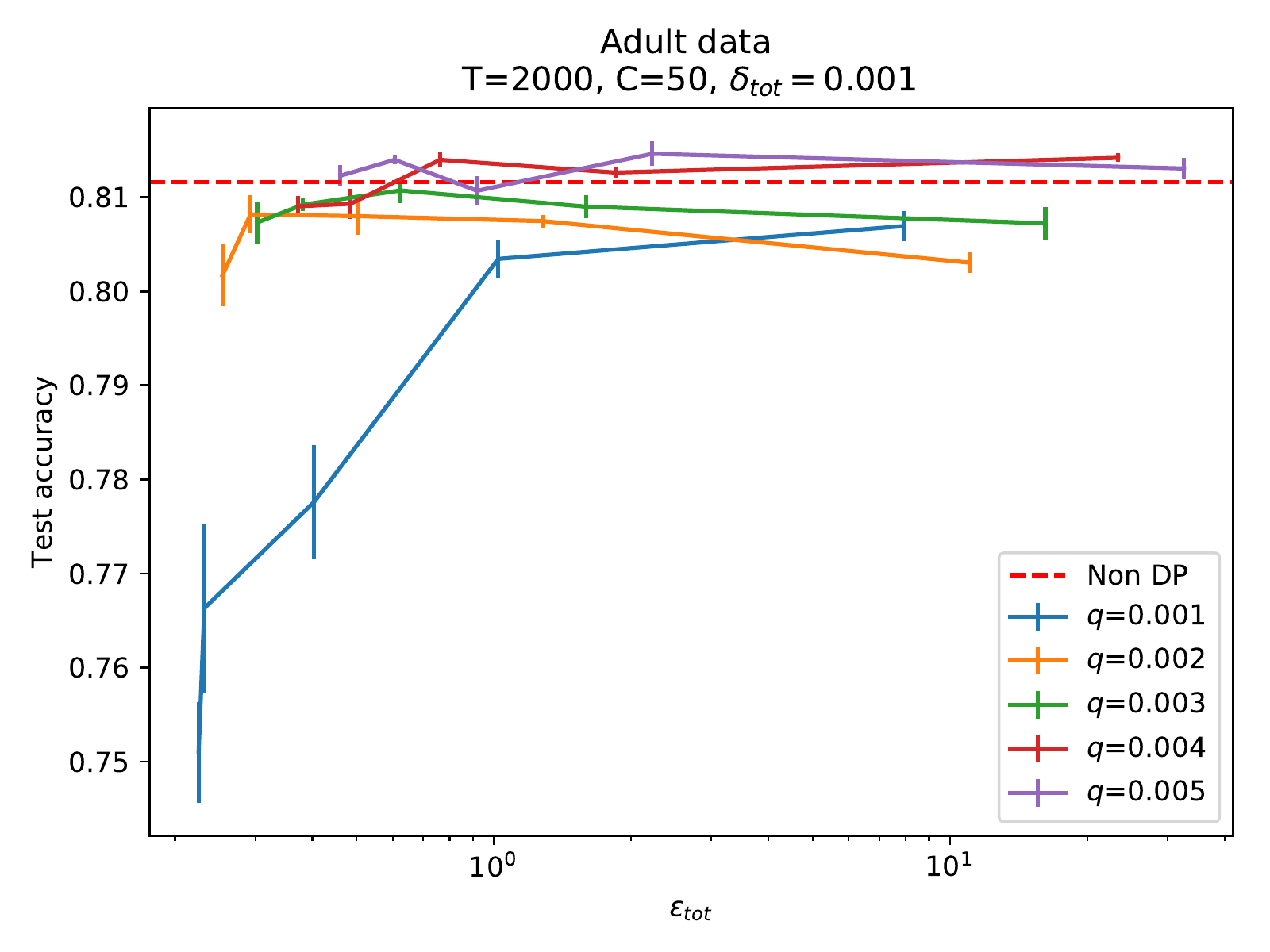}
  \caption{Accuracy vs.\ total $\epsilon$ in Abalone (left) and Adult
    (right) data sets with several data subsampling ratios $q$ in DPVI
    with the moments accountant. The curve shows the mean of 10 runs
    of the DP algorithm with error bars denoting the standard error of
    the mean.  Note that the $y$-axis scale covers a much smaller
    range than in Fig.~\ref{fig:dpvi_vs_sgld1}.}
  \label{fig:q_vs_epsilon}
\end{figure*}

Fig.~\ref{fig:q_vs_epsilon} shows the results when changing the data
subsampling rate $q$.  The result confirms the analysis of
Sec.~\ref{sec:select-algor-hyperp} that larger $q$ tend to perform
better than small $q$ although there is a limit how small values of
$\epsilon$ can be reached with a larger $q$.

\begin{figure*}[tb]
  \centering
  \includegraphics[width=\columnwidth]{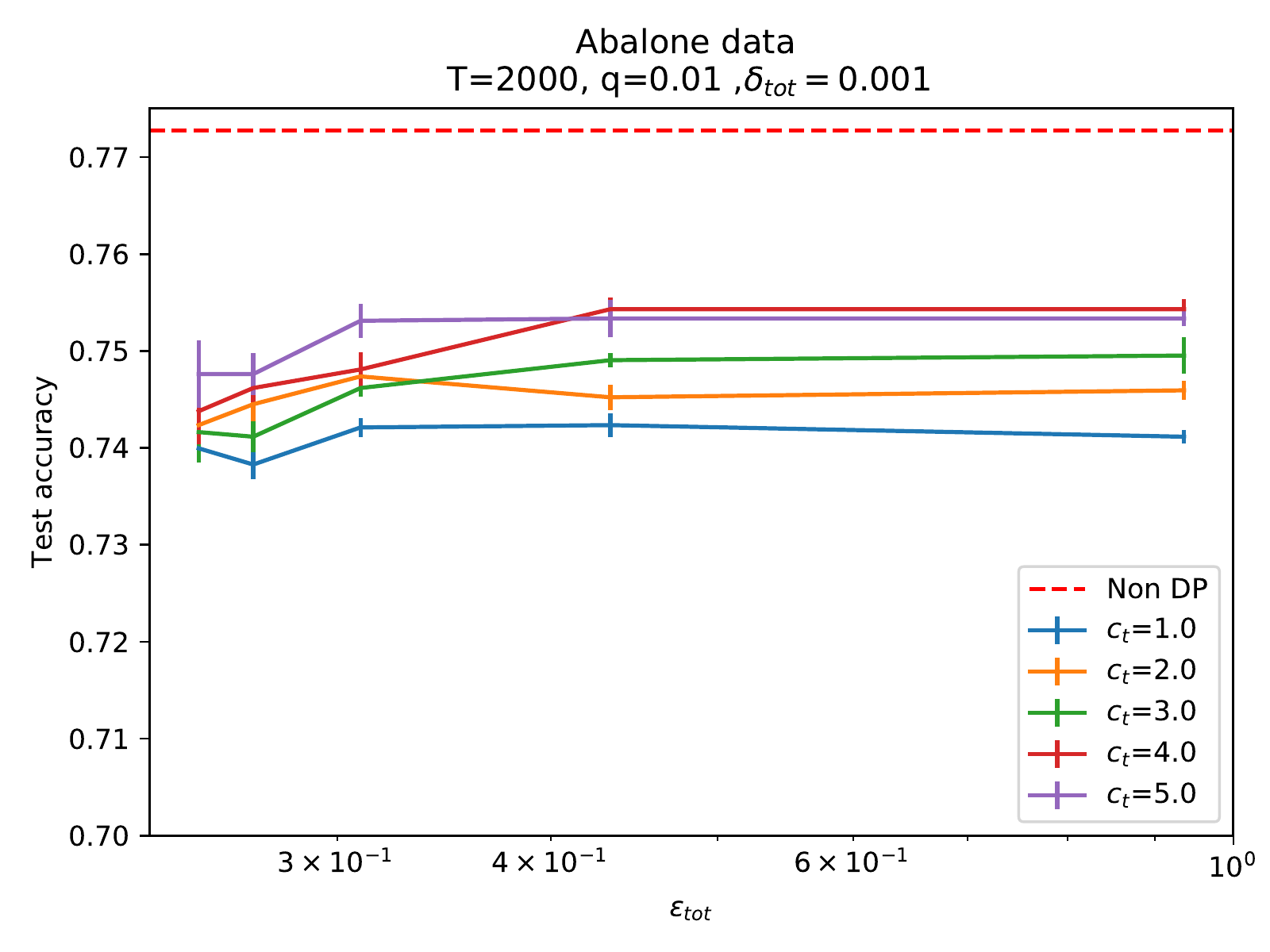}\hfill
  \includegraphics[width=\columnwidth]{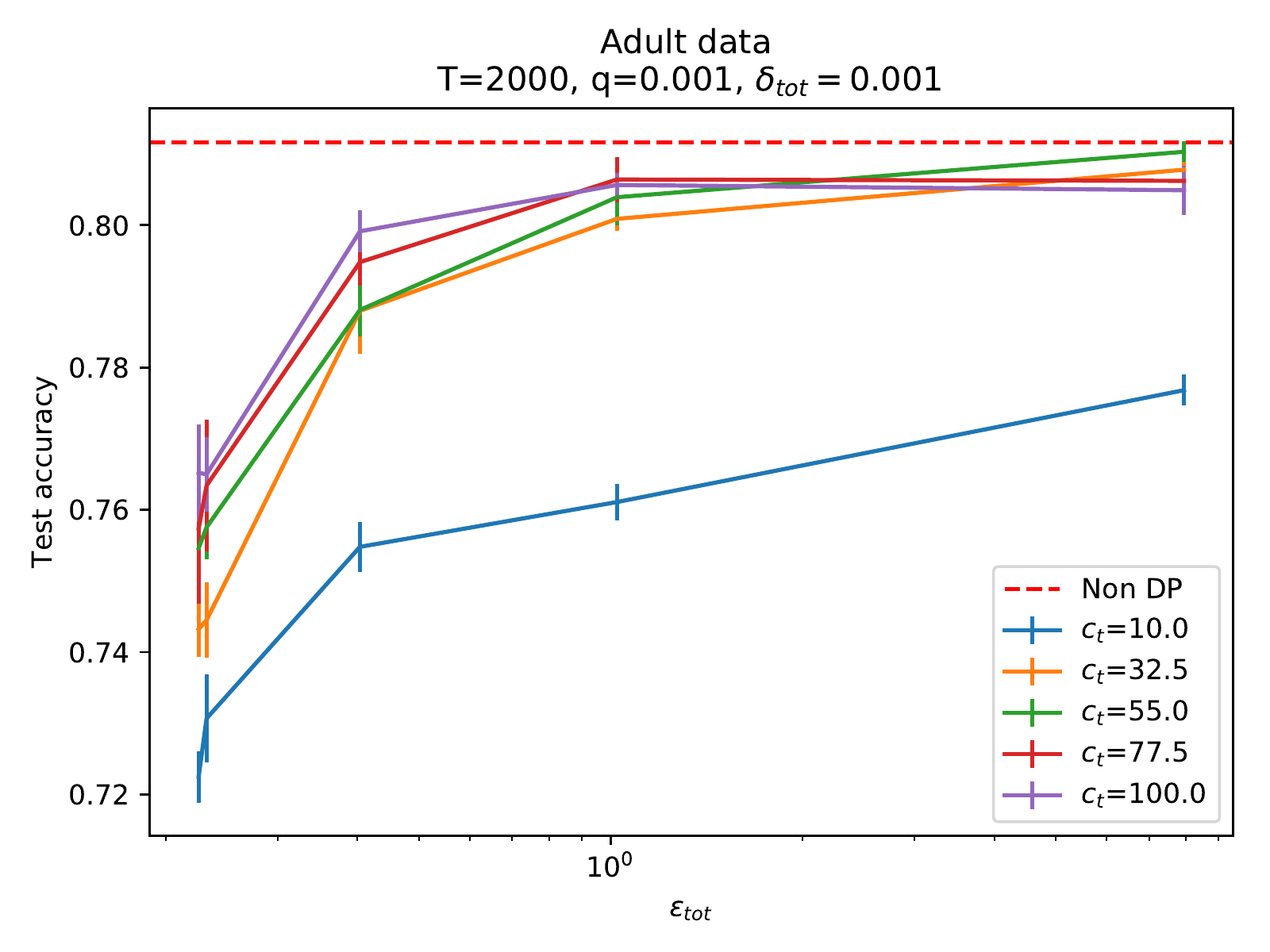}
  \caption{Accuracy vs.\ total $\epsilon$ in Abalone (left) and Adult
    (right) data sets with several gradient clipping threshold $c_t$
    values in DPVI with the moments accountant.
    The curve shows the mean of 10 runs of
    the DP algorithm with error bars denoting the standard error of
    the mean.  Note that the $y$-axis scale covers a much smaller
    range than in Fig.~\ref{fig:dpvi_vs_sgld1}.}
  \label{fig:c_vs_epsilon}
\end{figure*}

Fig.~\ref{fig:c_vs_epsilon} shows corresponding results when changing
the gradient clipping threshold $c_t$.  The results clearly show that
too strong clipping can hurt the accuracy significantly.  Once the
clipping is sufficiently mild the differences between different
options are far less dramatic.

\subsection{GAUSSIAN MIXTURE MODEL}

We also tested the performance of DPVI with a Gaussian mixture model.
For $K$ components our model is
\begin{align*}
\pi_k &\sim \text{Dir}(\alpha) \\
\mu^{(k)} &\sim \text{MVNormal}(\mathbf{0}, \mathbf{I}) \\
\tau^{(k)} &\sim \text{Inv-Gamma}(1,1) \\
\end{align*}
with the likelihood
\begin{align*}
p(\mathbf{x}_i |\boldsymbol\pi ,\boldsymbol\mu, \boldsymbol\tau) = \sum_{k=1}^K \pi_k \mathcal{N}(\mathbf{x_i}; \mu^{(k)}, \tau^{(k)} \mathbf{I}).
\end{align*}

Unlike standard variational inference that augments the model with
indicator variables denoting the component responsible for generating
each sample, we performed the inference directly on the mixture
likelihood.  This lets us avoid having to deal with latent variables
that would otherwise make the DP inference more complicated.

The posterior approximation
$q(\pi, \mu, \tau) = q(\pi) q(\mu) q(\tau)$ was fully factorised.
$q(\pi)$ was parametrised using softmax transformation from a diagonal
covariance Gaussian while $q(\mu)$ was Gaussian with a diagonal
covariance and $q(\tau)$ was log-normal with a diagonal covariance.

The synthetic data used in experiments was drawn from mixture of five 
spherical multivariate Gaussian distribution with means $[0,0], [\pm 2, \pm 2]$
and covariance matrices $\mathbf{I}$.  Similar data has been used
previously by \cite{Honkela2010JMLR} and \cite{Hensman2012}.
We used 1000 samples from this mixture for
training the model and 100 samples to test the performance.
We used both DPVI and DP-SGLD for this data. Performance comparison was done by
computing the predictive likelihoods for both algorithms with several different
epsilon values. 
We also show one example of approximate distribution that DPVI learns from above 
mixture model.

\begin{figure}
  \includegraphics[width=\columnwidth]{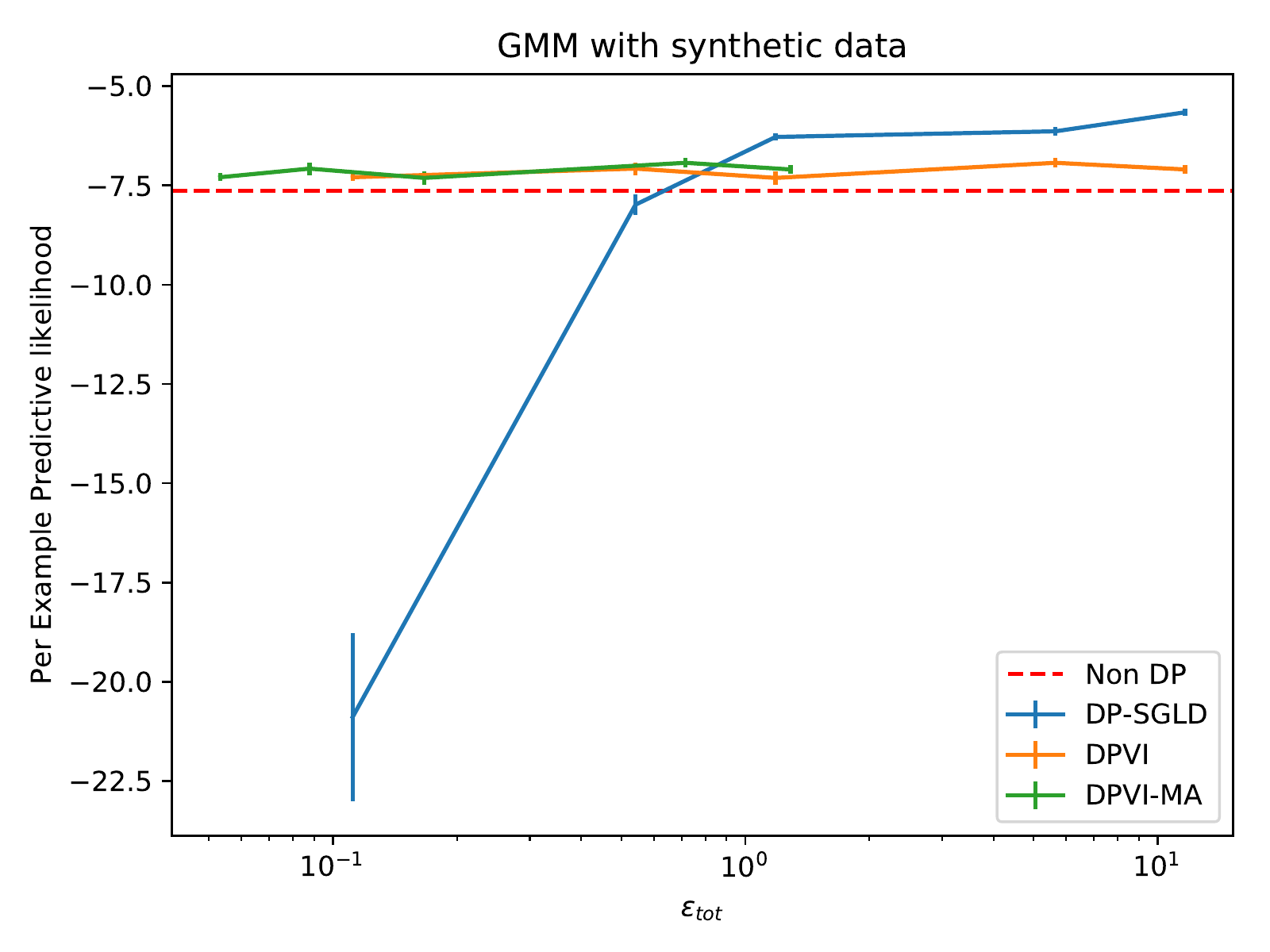}
  \caption{Per example predictive likelihood vs. $\epsilon$. For both DP-SGLD and
   DPVI, lines show mean between 10 runs of algorithm with error bars denoting the
   standard error of mean.}
  \label{fig:pred_like_comp}
\end{figure}

From Fig.~\ref{fig:pred_like_comp} we can see that DPVI algorithm performs well 
compared to non-private version of DPVI even with relatively small epsilon values.
However DP-SGLD seems to outperfom DPVI when we consider more flexible privacy requirements.
Both DP-SGLD and DPVI used $q=0.003$. We let
both algorithms run for 3000 iterations. Gradient clipping threshold for DPVI was set to $c_t = 5$.
We used $\delta = 0.001$ in the predictive likelihood comparison.
For DPVI predictive likelihood was approximated by Monte-Carlo integration using 
samples from the learned approximate posterior and for DP-SGLD by using the last 100 samples
the algorithm produced. 
Non-private results were obtained by setting $\sigma=0$ in DPVI, using $q=0.01$
and running the algorithm for 5000 iterations.

\begin{figure*}[tb]
  \includegraphics[width=\columnwidth]{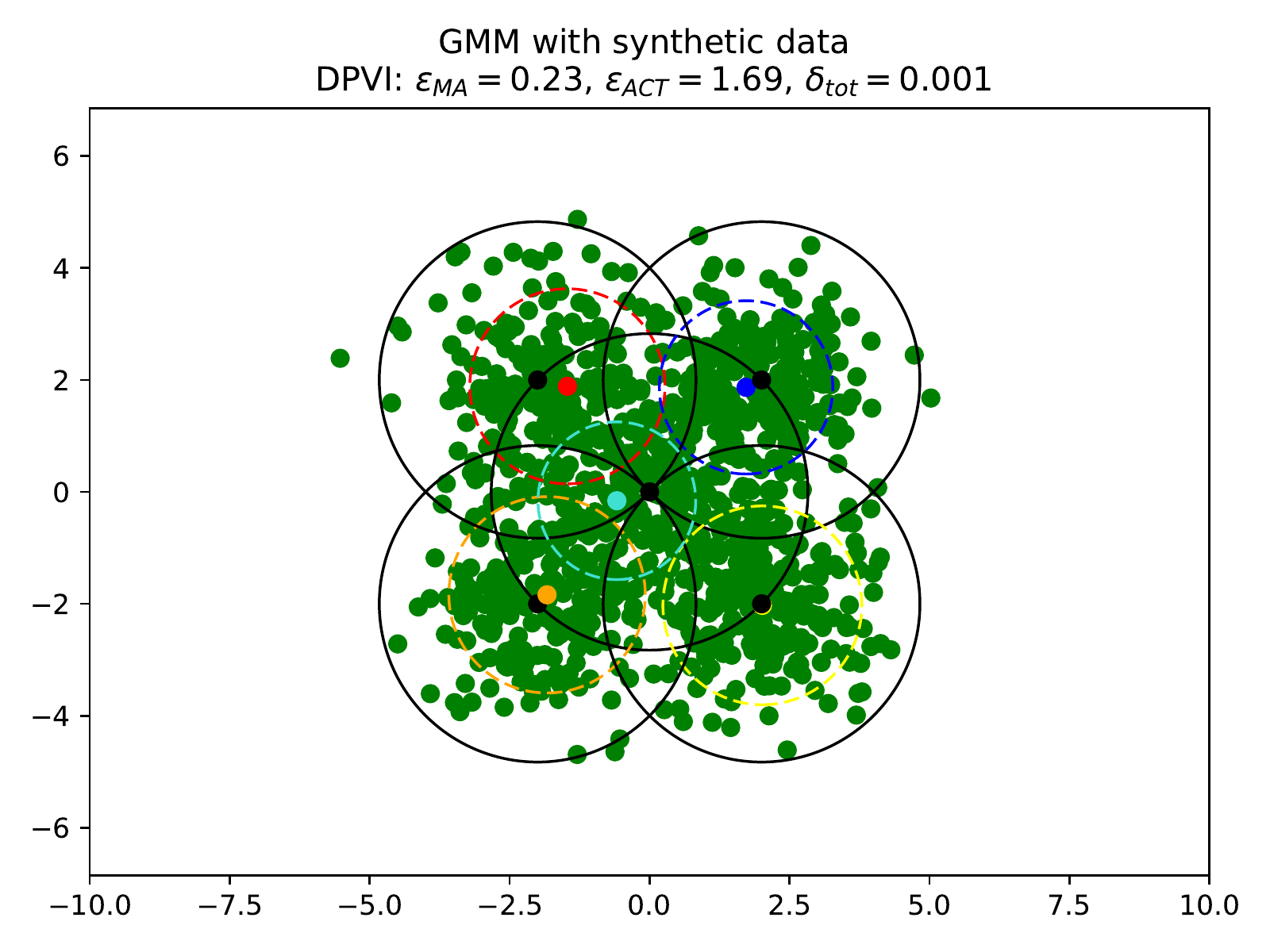}\hfill
  \includegraphics[width=\columnwidth]{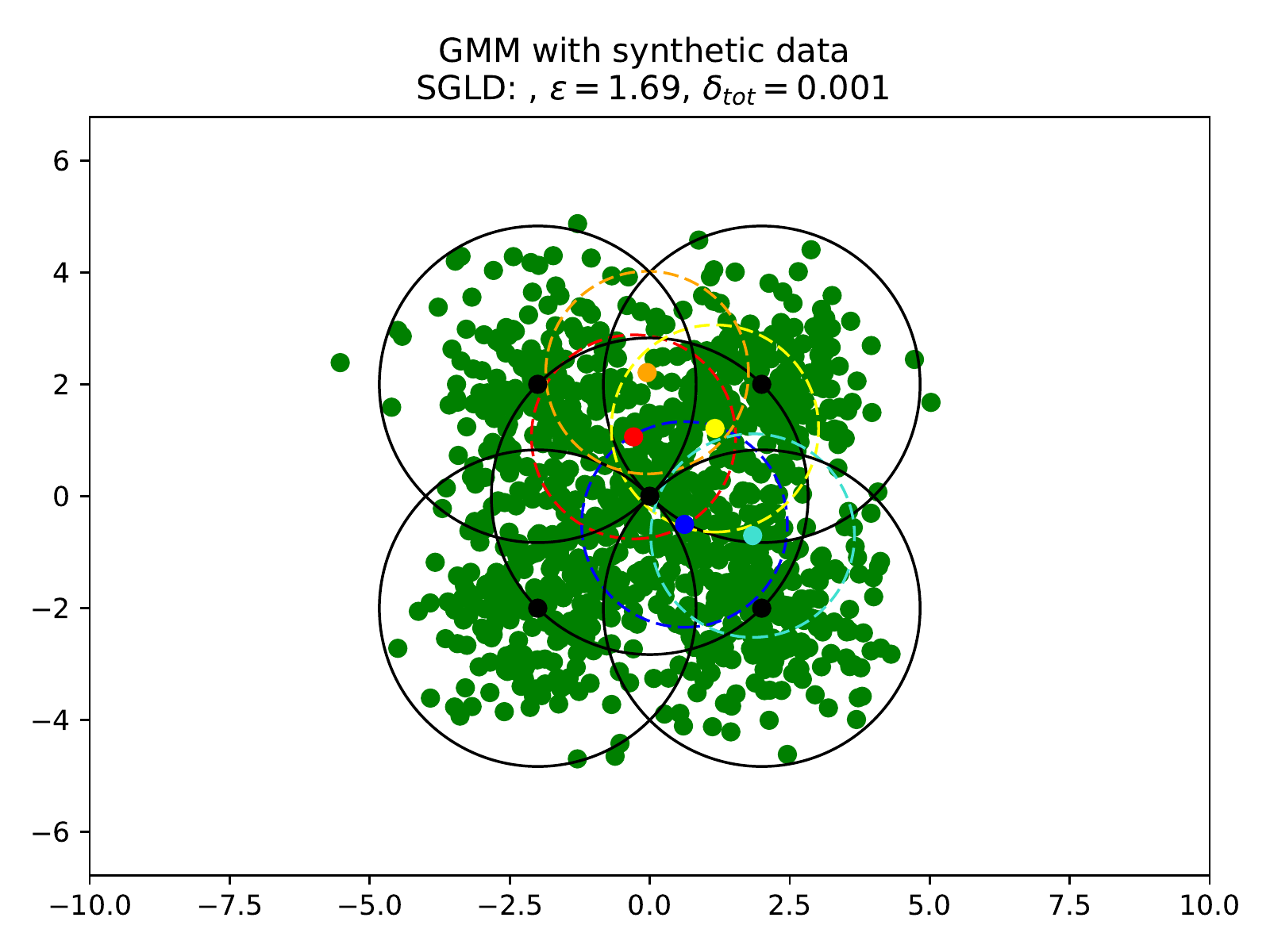}
  \caption{Approximate posterior predictive distribution for the
    Gaussian mixture model learned with DPVI (left) and DP-SGLD
    (right).  The DP-SGLD distribution is formed as an average
    over the last 100 samples from the algorithm.}
  \label{fig:gmm_dpvi}
\end{figure*}

Fig.~\ref{fig:gmm_dpvi} shows a visualisation of the mixture components
learned by DPVI and DP-SGLD.  Components inferred by DPVI appear much
closer to ground truth than those from DP-SGLD.


\section{DISCUSSION}

Our results demonstrate that the proposed DPVI method has the
potential to produce very accurate learning results, but this requires
finding good values for algorithmic hyperparameters that unfortunately
cannot be tuned using standard approaches without compromising the
privacy.  Finding good heuristics and default values for the
hyperparameters is a very important avenue of future research.

It is tempting to think that the effect of gradient clipping would
disappear as the algorithm converges and the total gradient becomes
smaller.  Unfortunately this is not true as the clipping is applied on
the level of data point specific gradients which will typically not
disappear even at convergence.  This also means that aggressive
clipping will change the stationary points of the SGA algorithm.

One way to make the problem easier to learn under DP is to simplify it
for example through dimensionality reduction.  This was noted for
exponential family models by \citet{Honkela2016} but the same
principle carries over to DPVI too.  In DPVI, lower dimensional model
typically has fewer parameters leading to a shorter parameter vector
whose norm would thus be smaller, implying that smaller $c_t$ is
enough.  This means that simpler posterior approximations such as
Gaussians with a diagonal covariance may be better under DP while
without the DP constraint an approximation with a full covariance
would usually be better \citep[see also][]{Kucukelbir2017}.


\section{CONCLUSIONS}
We have introduced the DPVI method that can deliver differentially
private inference results with accuracy close to the non-private
doubly stochastic variational inference and ADVI.
The method can effectively harness the power of ADVI to deal with very
general models instead of just conjugate exponential models and the
option of using multivariate Gaussian posterior approximations for
greater accuracy.


\subsubsection*{Acknowledgements}

This work was funded by the Academy of Finland (Centre of
Excellence COIN; and grants 278300, 259440 and 283107).


\bibliography{dpvb}
\bibliographystyle{myabbrvnat}
\end{document}